\documentclass{article}

\usepackage{amsmath}\allowdisplaybreaks
\usepackage{amssymb}
\usepackage{amsthm}
\usepackage{mathtools}
\usepackage{amsfonts}
\usepackage{algorithm}
\usepackage{algorithmic}
\usepackage{natbib}
\usepackage{bm}
\usepackage{verbatim}
\usepackage{lipsum}
\mathtoolsset{showonlyrefs,showmanualtags}

\usepackage{geometry}
\newtheorem{definition}{Definition}
\newtheorem{theorem}{Theorem}
\newtheorem{lemma}{Lemma}

\renewcommand{\cite}{\citep}
\newcommand{\bmx}{\bm{x}}
\newcommand{\bmy}{\bm{y}}
\newcommand{\bmf}{\bm{f}}

\newcommand{\LO}{\mathrm{LO}}
\newcommand{\LOmax}{\mathrm{LO}_{\max}}
\newcommand{\TZ}{\mathrm{TZ}}
\newcommand{\Omax}{\mathrm{O}_{\max}}
\newcommand{\CZ}{\mathrm{CZ}}
\newcommand{\CZmax}{\mathrm{CZ}_{\max}}
\newcommand\rank[1]{\mathrm{\texttt{rank}}(#1)}
\newcommand\dist[1]{\mathrm{\texttt{dist}}(#1)}

\AtBeginDocument{%
  \providecommand\BibTeX{{%
    \normalfont B\kern-0.5em{\scshape i\kern-0.25em b}\kern-0.8em\TeX}}}

\title{Running Time Analysis of the Non-dominated Sorting Genetic Algorithm II (NSGA-II) using Binary or Stochastic Tournament Selection}

\author{
	Chao Bian and	Chao Qian\thanks{Chao Qian is the corresponding author.}
	\\
	State  Key Laboratory for Novel Software Technology, \\Nanjing University, Nanjing 210023, China\\
	\{bianc,qianc\}@lamda.nju.edu.cn
}
\date 

\begin{document}

\maketitle

\begin{abstract}
	Evolutionary algorithms (EAs) have been widely used to solve multi-objective optimization problems, and have become the most popular tool. However, the theoretical foundation of multi-objective EAs (MOEAs), especially the essential theoretical aspect, i.e., running time analysis, has been still largely underdeveloped. The few existing theoretical works mainly considered simple MOEAs, while the non-dominated sorting genetic algorithm II (NSGA-II), probably the most influential MOEA, has not been analyzed except for a very recent work considering a simplified variant without crossover. In this paper, we present a running time analysis of the standard NSGA-II for solving LOTZ, OneMinMax and COCZ, the three commonly used bi-objective optimization problems. Specifically, we prove that the expected running time (i.e., number of fitness evaluations) is $O(n^3)$ for LOTZ, and $O(n^2\log n)$ for OneMinMax and COCZ, which is surprisingly as same as that of the previously analyzed simple MOEAs, GSEMO and SEMO. Next, we introduce a new parent selection strategy, stochastic tournament selection (i.e., $k$ tournament selection where $k$ is uniformly sampled at random), to replace the binary tournament selection strategy of NSGA-II, decreasing the required expected running time to $O(n^2)$ for all the three problems. Experiments are also conducted, suggesting that the derived running time upper bounds are tight for LOTZ, and almost tight for OneMinMax and COCZ.
\end{abstract}

\section{Introduction}
 
 Multi-objective optimization~\cite{Steuer86}, which requires to optimize several objective functions simultaneously, arises in many areas. Since the objectives are usually conflicting, there doesn't exist a single solution which can perform well on all these objective functions. Thus, 
 the goal of multi-objective optimization is to find a set of Pareto optimal solutions (or the Pareto front), representing different optimal trade-offs between these objectives. 
 Evolutionary algorithms (EAs)~\cite{back:96} are a kind of randomized heuristic optimization algorithms, inspired by natural evolution. They maintain a set of solutions (i.e., a population), and iteratively improve the population by reproducing new solutions and selecting better ones. Due to their population-based nature, EAs are very popular for solving multi-objective optimization problems, and have been widely used in many real-world applications~\cite{coello2004applications}.
 
 Compared with practical applications, the theoretical foundation of EAs is still underdeveloped, which is mainly because the sophisticated behaviors of EAs make theoretical analysis quite difficult. Though much effort has been devoted to the essential theoretical aspect, i.e., running time analysis, leading to a lot of progresses~\cite{neumann2010bioinspired,auger2011theory,doerr-telo21-survey} in the past 25 years, most of them focused on single-objective optimization, while only a few considered the more complicated scenario of multi-objective optimization. Next, we briefly review the results of running time analyses on multi-objective EAs (MOEAs). 
 
 The running time analysis of MOEAs started from GSEMO, a simple MOEA which employs the bit-wise mutation operator to generate an offspring solution in each iteration and keeps the non-dominated solutions generated-so-far in the population. 
For GSEMO solving the bi-objective optimization problems LOTZ and COCZ, the expected running time has been proved to be $O(n^3)$~\cite{Giel03} and $O(n^2 \log n)$~\cite{Qian13,bian2018tools}, respectively, where $n$ is the problem size.
SEMO is a counterpart of GSEMO, which employs the local mutation operator, one-bit mutation, instead of the global bit-wise mutation operator. \citet{LaumannsTEC04} proved that the  expected running time of SEMO solving LOTZ and COCZ are $\Theta(n^3)$ and $O(n^2\log n)$, respectively. 
\citet{Giel10} considered another bi-objective problem OneMinMax, 
and proved that both GSEMO and SEMO can solve it in $O(n^2\log n)$ expected running time. 
\citet{doerr2013lower} also proved a lower bound $\Omega(n^2/p)$ for GSEMO solving  LOTZ, where $p<n^{-7/4}$ is the mutation rate, i.e., the probability of flipping each bit when performing bit-wise mutation.

Later, the analyses of GSEMO were conducted on multi-objective combinatorial optimization problems. 
For bi-objective minimum spanning trees (MST), GSEMO was proved to be able to find a 2-approximation of the Pareto front in expected pseudo-polynomial time~\cite{Neumann07}. 
For multi-objective shortest paths, a variant of GSEMO can achieve an $(1+\epsilon)$-approximation in expected pseudo-polynomial time~\cite{Horoba09,Neumann10}, where $\epsilon>0$. 
\citet{laumanns-nc04-knapsack} considered the GSEMO and its variant for solving a special case of the multi-objective knapsack problem, and proved that the expected running time of the two algorithms for finding all the Pareto optimal solutions are $O(n^6)$ and $O(n^5)$, respectively. 
  
  There are also studies that analyze the GSEMO for solving single-objective constrained optimization problems. By optimizing a reformulated bi-objective optimization problem that optimizes the original objective and a constraint-related objective simultaneously, the GSEMO can reduce the expected running time significantly for achieving a desired approximation ratio. For example, by reformulating the MST problem into a bi-objective problem, \citet{neumann-nc06-mst-easier} proved that GSEMO and SEMO can solve MST in $O(mn(n+\log w_{\max}))$ expected running time, which is better than the $O(m^2(\log n+\log w_{\max}))$ expected running time achieved by (1+1)-EA and RLS~\cite{neumann-tcs07-mst}, i.e., single-objective counterparts to GSEMO and SEMO, for $m=\Theta(n^2)$, where $m,n$ and $w_{\max}$ denote the number of edges, the number of nodes and the largest weight of the graph, respectively.
  More evidences have been proved on the problems of minimum cuts~\cite{neumann-algo11-cut}, set cover~\cite{friedrich-ecj10-cover} and submodular optimization~\cite{friedrich2015maximizing}. Note that we concern multi-objective optimization problems in this paper.
  
  Based on the GSEMO and SEMO, the effectiveness of  some strategies for multi-objective evolutionary optimization have been analyzed.  For example, 
   \citet{LaumannsTEC04} showed the effectiveness of greedy selection by proving that using this strategy can reduce the expected running time of SEMO from $O(n^2\log n)$ to $\Theta(n^2)$ for solving the COCZ problem.
  \citet{Qian13} showed that crossover can accelerate filling the Pareto front by comparing the expected running time of GSEMO with and without crossover for solving the artificial problems COCZ and weighted LPTNO (a generalization of LOTZ), as well as the  combinatorial problem multi-objective MST.
  The effectiveness of some other mechanisms, e.g., heuristic selection~\cite{qian-ppsn16-hyper}, diversity~\cite{plateaus10}, fairness~\cite{LaumannsTEC04,friedrich2011illustration}, and diversity-based parent selection~\cite{osuna2020diversity} have also been examined. 
  
  Though the GSEMO and SEMO share the general structure of MOEAs, they have been much simplified. To characterize the behaviors of practical MOEAs, some efforts have been devoted to analyzing MOEA/D, which is a popular MOEA based on decomposition~\cite{zhang2007moea}.
  \citet{li2015primary} analyzed a simplified variant of MOEA/D without crossover for solving COCZ and weighted LPTNO, and proved that the expected running time is $\Theta(n\log n)$ and $\Theta(n^2)$, respectively. 
  \citet{huang2021runtime} also considered a simplified MOEA/D, and examined the effectiveness of different decomposition approaches by comparing the running time for solving two many-objective problems $m$LOTZ and $m$COCZ,  where $m$ denotes the number of objectives.
 
 Surprisingly, the running time analysis of the non-dominated sorting genetic algorithm II (NSGA-II)~\cite{deb-tec02-nsgaii}, the probably most influential MOEA, has been rarely touched. The NSGA-II enables to find well-spread Pareto-optimal solutions by incorporating two substantial features, i.e., non-dominated sorting and crowding distance, and has become the most popular MOEA for solving multi-objective optimization problems~\cite{deb2011multi}.
 To the best of our knowledge, the only attempt is a very recent work, which, however, considered a simplified version of NSGA-II without crossover, and proved that the expected running time is $O(n^2 \log n)$ for OneMinMax and $O(n^3)$ for LOTZ~\cite{zheng2021first}.
  
 In this paper, we present a running time analysis for the standard NSGA-II. We prove that for NSGA-II solving LOTZ, the expected running time is $O(n^3)$; while for OneMinMax and COCZ, the expected running time is $O(n^2\log n)$. Note that these running time upper bounds are as same as that of GSEMO and SEMO~\cite{LaumannsTEC04,Giel03,Qian13,Giel10}, implying that the NSGA-II does not have advantage over simplified MOEAs on these problems if the derived upper bounds are tight.
 
 Next, we introduce a new parent selection strategy, i.e., stochastic tournament selection, which samples a number $k$ uniformly at random and then performs $k$ tournament selection. By replacing the original binary tournament selection of NSGA-II with stochastic tournament selection, we prove that the expected running time of NSGA-II can be improved to $O(n^2)$ for LOTZ, OneMinMax and COCZ. 
 We also conduct experiments to show that the derived upper bounds are almost tight.
 The goal of this work is to take a step towards analyzing the running time of practical MOEAs, and meanwhile, the introduced stochastic tournament selection strategy may be helpful in practical applications.
 
\section{Preliminaries}
In this section, we first introduce multi-objective optimization and the procedure of NSGA-II, and then present a new tournament selection strategy, i.e., stochastic tournament selection.

\subsection{Multi-objective Optimization}

Multi-objective optimization requires to simultaneously optimize two or more objective functions, as shown in Definition~\ref{def_MO}. We consider maximization here, while minimization can be defined similarly. The objectives are usually conflicting, and thus there is no canonical complete order in the solution space $\mathcal{X}$. 
The comparison between solutions relies on the \emph{domination} relationship, as presented in Definition~\ref{def_Domination}. A solution is \emph{Pareto optimal} if there is no other solution in $\mathcal{X}$ that dominates it. The set of objective vectors of all the Pareto optimal solutions constitutes the \emph{Pareto front}. The goal of multi-objective optimization is to find the Pareto front, that is, to find at least one corresponding solution for each objective vector in the Pareto front.

\begin{definition}[Multi-objective Optimization]\label{def_MO}
	Given a feasible solution space $\mathcal{X}$ and objective functions $f_1,f_2,\ldots, f_m$, multi-objective optimization can be formulated as\vspace{-0.3em}
	\begin{align}
		\max\nolimits_{\bmx \in
			\mathcal{X}}\; \big(f_1(\bmx),f_2(\bmx),...,f_m(\bmx)\big).
	\end{align}
\end{definition}
\begin{definition}[Domination]\label{def_Domination}
	Let $\bm f = (f_1,f_2,\ldots, f_m):\mathcal{X} \rightarrow \mathbb{R}^m$ be the objective vector. For two solutions $\bmx$ and $\bmy\in \mathcal{X}$:\vspace{-0.1em}
	\begin{itemize}
		\item $\bmx$ \emph{weakly dominates} $\bmy$  (denoted as $\bmx \succeq \bmy$) if $\forall 1 \leq i \leq m, f_i(\bmx) \geq f_i(\bmy)$;
		\item $\bmx$ \emph{dominates} $\bmy$ (denoted as $\bmx\succ \bmy$) if $\bm{x} \succeq \bmy$ and $f_i(\bmx) > f_i(\bmy)$ for some $i$;
		\item  $\bmx$ and $\bmy$ are \emph{incomparable} if neither $\bmx\succeq \bmy$ nor $\bmy\succeq \bmx$.
	\end{itemize}
\end{definition}

\subsection{NSGA-II}\label{sec:nsgaii}
 The NSGA-II algorithm~\cite{deb-tec02-nsgaii} as presented in Algorithm~\ref{alg:nsgaii} is a popular MOEA, which  incorporates two substantial features, i.e., non-dominated sorting in Algorithm~\ref{alg:fastsort}, and crowding distance in Algorithm~\ref{alg:crowdist}. 
 NSGA-II starts from an initial population of $N$ random solutions (line~1). In each generation, it  employs binary tournament selection $N$ times to generate a parent population $P'$ (line~4),  and then applies one-point crossover and bit-wise mutation on the $N/2$ pairs of parent solutions to generate $N$  offspring solutions (lines~5--9).  
 Note that the two adjacent selected solutions form a pair, and thus the $N$ selected solutions form $N/2$ pairs.
 The one-point crossover operator first selects a  crossover point $i\in \{1,2,\ldots,n\}$ uniformly at random,  where $n$ is the problem size, and then exchanges the first $i$ bits of two solutions. The bit-wise mutation operator flips each bit of a solution independently with probability $1/n$.
 The binary tournament selection presented in Definition~\ref{def:bin-tour} picks two solutions randomly from the population $P$ with or without replacement, and then selects a better one (ties broken uniformly). Note that we consider the strategy with replacement in this paper.
 \begin{algorithm}[t]
 	\caption{NSGA-II Algorithm~\cite{deb-tec02-nsgaii}}
 	\label{alg:nsgaii}
 	\begin{flushleft}
 		\textbf{Input}: objective functions $f_1,f_2\ldots,f_m$, population size $N$\\
 		\textbf{Output}: $N$ solutions from $\{0,1\}^n$\\
 		\textbf{Process}:
 	\end{flushleft}
 	\begin{algorithmic}[1]
 		\STATE $P\!\leftarrow\!\! N$ solutions uniformly and randomly selected from $\{0,\! 1\}^{\!n}$;
 		\WHILE{criterion is not met}
 		\STATE $Q=\emptyset$;  
 		\STATE apply binary tournament selection $N$ times to generate a parent population $P'$ of size $N$;
 		\FOR{each pair of the parent solutions $\bmx$ and $\bmy$ in $P'$}
 		\STATE apply one-point crossover on $\bmx$ and $\bmy$ to generate two solutions $\bmx'$ and $\bmy'$, with probability 0.9;
 		\STATE apply bit-wise mutation on $\bmx'$ and $\bmy'$ to generate $\bmx''$ and $\bmy''$, respectively;
 		\STATE add $\bmx''$ and $\bmy''$ into $Q$
 		\ENDFOR
 		\STATE apply Algorithm~\ref{alg:fastsort} to partition $P\cup Q$ into non-dominated sets $F_1,F_2,\ldots$;
 		\STATE let $P=\emptyset$, $i=1$;
 		\WHILE{$|P\cup F_i|<N$}
 		\STATE $P=P\cup F_i$, $i=i+1$ 
 		\ENDWHILE
 		\STATE  apply Algorithm~\ref{alg:crowdist} to assign each solution in $F_i$ with a crowding distance; 
 		\STATE sort the solutions in $F_i$ by crowding distance in descending order, and add the first $N-|P |$ solutions into $P$
 		\ENDWHILE
 		\RETURN $P$
 	\end{algorithmic}
 \end{algorithm}
 \begin{definition}[Binary Tournament Selection]\label{def:bin-tour}
 	The binary tournament selection strategy first picks two solutions from the population $P$ uniformly at random, and then selects a better one with ties broken uniformly.
 \end{definition}
  After generating $N$ offspring solutions, the best $N$ solutions in the current population $P$ and the offspring population $Q$ are selected as the population in the next generation (lines~10--16).  
 In particular,  the solutions in the current and offspring populations are partitioned into non-dominated sets $F_1,F_2,\ldots$ (line~10), where $F_1$ contains all the non-dominated solutions in $P \cup Q$, and $F_i$ ($i\ge 2$) contains all the non-dominated solutions in $(P \cup Q) \setminus \cup_{j=1}^{i-1} F_j$.  The fast implementation of not-dominated sorting is  presented in Algorithm~\ref{alg:fastsort}.   Not that we use the notion $\rank{\bmx}=i$ to denote that $\bmx$ belongs to $F_i$.
 Then, the solutions in $F_1,F_2,\ldots$ are added into the next population (lines~12--14),  until the population size exceeds $N$.   For the critical set $F_i$,  i.e., the inclusion of which can make the population size larger than $N$, Algorithm~\ref{alg:crowdist} is used to compute the crowding distance for each of the solutions in it (line~15). Finally, the solutions in $F_i$ with large crowding distance are selected to fill the remaining population slots (line~16). 
\begin{algorithm}[t]
	\caption{Fast Non-dominated Sorting~\cite{deb-tec02-nsgaii}}
	\label{alg:fastsort}
	\begin{flushleft}
		\textbf{Input}:  a population $P$\\
		\textbf{Output}: non-dominated sets $F_1,F_2,\ldots$\\
		\textbf{Process}:
	\end{flushleft}
	\begin{algorithmic}[1]
		\STATE $F_1=\emptyset$;
		\FOR{ each $\bmx \in P$}
		\STATE $S_{\bmx}=\emptyset ; n_{\bmx}=0$;
		\FOR{ each $y \in P$}
		\IF{$\bmx\succ \bmy$}
		\STATE $S_{\bmx}=S_{\bmx} \cup\{\bmy\}$
		\ELSIF {$\bmx \prec \bmy$}
		\STATE $n_{\bmx}=n_{\bmx}+1$
		\ENDIF
		\ENDFOR
		\IF{$n_{\bmx}=0$}
		\STATE $\rank{\bmx}=1 ; F_{1}=F_{1}\cup \{\bmx\}$
		\ENDIF
		\ENDFOR
		\STATE $i=1$;
		\WHILE{ $F_{i} \neq \emptyset$}
		\STATE $Q=\emptyset$;
		\FOR{ each $\bmx \in F_{i}$ }
		\FOR{ each $\bmy \in S_{\bmx}$ }
		\STATE $n_{\bmy}=n_{\bmy}-1$;
		\IF{ $n_{\bmy}=0$}
		\STATE $\rank{\bmy}=i+1$; $Q=Q \cup\{\bmy\}$
		\ENDIF
		\ENDFOR
		\ENDFOR
		\STATE $i=i+1$; $F_i=Q$
		\ENDWHILE
	\end{algorithmic}
\end{algorithm}
\begin{algorithm}[t]
	\caption{Crowding Distance Assignment~\cite{deb-tec02-nsgaii}}
	\label{alg:crowdist}
	\begin{flushleft}
		\textbf{Input}: $Q=\{\bmx^1,\bmx^2,\ldots,\bmx^l\}$ with the same rank\\
		\textbf{Output}: the crowding distance $\dist{\cdot}$ for each solution in $Q$\\
		\textbf{Process}:
	\end{flushleft}
	\begin{algorithmic}[1]
		\STATE let $\dist{\bmx^j}=0$ for any $j\in\{1,2,\ldots,l\}$;
		\FOR{$i=1$ to $m$}
		\STATE sort the solutions in $Q$ w.r.t. $f_i$ in ascending order;
		\STATE $\dist{Q[1]}=\dist{Q[l]}=\infty$;
		\FOR{$j=2$ to $l-1$}
		\STATE $\dist{Q[j]}=\dist{Q[j]}+\frac{f_i(Q[j+1])-f_i(Q[j-1])}{f_i(Q[l])-f_i(Q[1])}$
		\ENDFOR
		\ENDFOR
	\end{algorithmic}
\end{algorithm}

 When using binary tournament selection (line~4), the selection criterion is based on the crowded-comparison, that is, a solution $\bmx$ is superior to $\bmy$ (denoted as $\bmx\succ_{\mathrm{c}}\bmy$) if 
\begin{align}\label{eq:crowd-comp}
	&\rank{\bmx}<\rank{\bmy} \text{ or }\\	&\rank{\bmx}=\rank{\bmy}\wedge \dist{\bmx}>\dist{\bmy}.
\end{align}
Intuitively, the crowding distance of a solution  means the distance between its closest neighbour solutions, and a solution with larger  crowding distance is preferred so that the diversity of the population can be preserved as much as possible.
Note that in Algorithm~\ref{alg:crowdist}, we assume that the relative positions of the solutions with the same objective vector are unchanged or  totally reversed when the solutions are sorted  w.r.t. some objective function (line~3).

In line~6 of Algorithm~\ref{alg:nsgaii}, the probability of using crossover has been set to 0.9, which is the same as the original setting and also commonly used~\cite{deb-tec02-nsgaii}. However, the theoretical results derived in this paper can be directly generalized to the scenario where the  probability of using crossover belongs to $[\Omega(1),1-\Omega(1)]$.

\subsection{Stochastic Tournament Selection}
As the crowded-comparison $\succ_{\mathrm{c}}$ in Eq.~\eqref{eq:crowd-comp} actually gives a total order of the solutions in the population $P$,
binary tournament selection can be naturally extended to $k$ tournament selection~\cite{eiben-book}, as presented in Definition~\ref{def:k-tour}, where $k$ is a parameter such that $1\le k\le N$. 
That is, $k$ solutions are first picked from $P$ uniformly at random, and then the solution with the smallest rank is selected. If several solutions have the same smallest rank, the one with the largest crowding distance is selected, with ties broken uniformly. 
\begin{definition}[$k$ Tournament Selection]\label{def:k-tour}
	The $k$ tournament selection strategy first picks $k$ solutions from the population $P$ uniformly at random, and then selects the best one with ties broken uniformly.
\end{definition}
Note that a larger $k$ implies a larger selection pressure, i.e., a larger probability of selecting a good solution, and thus the value of $k$ can be used to control the selection pressure of EAs~\cite{eiben-book}.  However, this also brings about a new issue, i.e.,  how to set $k$ properly.
In order to reduce the risk of setting improper values of $k$ as well as the overhead of tuning $k$, we introduce a natural strategy, i.e.,  stochastic tournament selection in Definition~\ref{def:sto-tour}, which first selects a number $k$ randomly, and then performs the $k$ tournament selection. In this paper, we consider that the tournament candidates are picked with replacement from the population.
\begin{definition}[Stochastic Tournament Selection]\label{def:sto-tour}
	The stochastic tournament selection strategy first selects a number $k$ from $\{1,2,\ldots,N\}$ uniformly at random, where $N$ is the size of the population $P$, and then employs the $k$ tournament selection to select a solution from the  population $P$.
\end{definition}
In each generation of NSGA-II, we need to select $N$ parent solutions independently, and each selection may involve the comparison of several solutions, which may lead to a large number of comparisons. To improve the efficiency of stochastic tournament selection, we can first sort the solutions in the population $P$, and then perform the parent selection procedure. Specifically, each solution $\bmx_i$ ($1\le i\le N$) in $P$ is assigned a number $\pi(i)$, where $\pi: \{1,2,\ldots,N\}\rightarrow \{1,2,\ldots,N\}$ is a bijection such that 
\begin{equation}\label{eq:pi}
	\forall 1\le i,j\le N, i\neq j: \bmx_i\succ_{\mathrm{c}} \bmx_j\Rightarrow \pi(i)<\pi(j).
\end{equation} 
That is, a solution with a smaller number is better. 
Note that the number $\pi(\cdot)$ is assigned randomly if several solutions have the same rank and crowding distance.
Then, we sample a number $k$ randomly from $\{1,2,\ldots,N\}$ and pick $k$ solutions from $P$ at random, where the solution with the lowest $\pi(\cdot)$ value is finally selected. 

Lemma~\ref{lem:stotour-prob} presents the property of stochastic tournament selection, which will be used in the following theoretical analysis.
It shows that any solution (even the worst solution) in $P$ can be selected with probability at least $1/N^2$, and any solution belonging to the best $O(1)$ solutions in $P$ (with respect to $\succ_{\mathrm{c}}$) can be selected with probability at least $\Omega(1)$. 
Note that for binary tournament selection, the probability of selecting the worst solution (denoted as $\bmx^{\mathrm{w}}$) is $1/N^2$, because $\bmx^{\mathrm{w}}$ is selected if and only if the two solutions picked for competition are both $\bmx^{\mathrm{w}}$; the probability of selecting the best solution (denoted as $\bmx^{\mathrm{b}}$) is $1-(1-1/N)^2=2/N-1/N^2$, because $\bmx^{\mathrm{b}}$ is selected if and only if $\bmx^{\mathrm{b}}$ is picked at least once.
Thus, compared with binary tournament selection,  stochastic tournament selection can increase the probability of selecting the top solutions, and meanwhile maintaining the probability of selecting the bottom solutions. 
\begin{lemma}\label{lem:stotour-prob}
	If using stochastic tournament selection, any solution in $P$ can be selected with prob. at least $1/N^2$. 
	Furthermore, a solution $\bmx_i\in P$ with $\pi(i)\!=\!O(1)$  can be selected with prob. $\Omega(1)$, where $\pi: \{1,2,\ldots,N\}\rightarrow \{1,2,\ldots,N\}$ is  a bijection satisfying Eq.~\eqref{eq:pi}.
\end{lemma}
\begin{proof}
	For any  solution $\bmx \in P$, it can be selected if $k=1$ and the solution picked for competition is exactly $\bmx$.  The probabilities of the two events are both $1/N$, implying a lower bound $1/N^2$ on the probability of selecting $\bmx$ as a parent solution.
	For the furthermore clause, we consider the case that $k\ge N/2$. Suppose that $\bmx_i$  is a solution with $\pi(i)=O(1)$. Then, it can be selected if $\bmx_i$ is picked for competition, while any solution $\bmx_j$ with $\pi(j)<\pi(i)$ is not picked.  The probability of not picking any $\bmx_j$ with $\pi(j)<\pi(i)$ is $(1-(\pi(i)-1)/N)^k$, and conditional on this event, the probability of picking $\bmx_i$ is $1-(1-1/(N-\pi(i)+1))^k$.
	Thus,  the probability of selecting $\bmx_i$ given $k\ge N/2$ is
	\begin{align}
		&\bigg(1\!-\!\frac{\pi(i)-1}{N}\bigg)^k \cdot \bigg(1-\bigg(1-\frac{1}{N-\pi(i)+1}\bigg)^k\bigg)\\
		&=\bigg(1-\frac{\pi(i)-1}{N}\bigg)^k - \bigg(\frac{N-\pi(i)}{N}\bigg)^k \\
		&\ge \binom{k}{0}\bigg(1-\frac{\pi(i)}{N}\bigg)^k\bigg(\frac{1}{N}\bigg)^0 \!+\binom{k}{1}\bigg(1-\frac{\pi(i)}{N}\bigg)^{k-1}\bigg(\frac{1}{N}\bigg)  \!-\!\bigg(1\!-\!\frac{\pi(i)}{N}\bigg)^k\\
		&= \frac{k}{N}\bigg(1-\frac{1}{N/\pi(i)}\bigg)^{(N/\pi(i)-1)\cdot (k-1)/(N/\pi(i)-1)}\\
		&		\ge \frac{k}{N}\cdot \bigg(\frac{1}{e}\bigg)^{(k-1)/(N/\pi(i)-1)}=\Omega(1),
	\end{align}
	where the last equality is by $N/2\le k\le N$ and $\pi(i)=O(1)$. Note that the probability of selecting a $k$ such that $k\ge N/2$ is 1/2, and thus the lemma holds.
\end{proof}


\section{Running Time Analysis of NSGA-II}
In this section, we analyze the expected running time of the standard NSGA-II in Algorithm~\ref{alg:nsgaii}, i.e., NSGA-II using binary tournament selection, solving three bi-objective pseudo-Boolean problems LOTZ, OneMinMax and COCZ, which are widely used in MOEAs’ theoretical analyses~\cite{Giel03,LaumannsTEC04,Giel10,doerr2013lower,Qian13}.

The LOTZ  problem presented in Definition~\ref{def:LOTZ} aims to maximize the number of leading 1-bits and the number of trailing 0-bits of a binary bit string.  
The Pareto front of LOTZ is $\mathcal{F}=\{(0,n),(1,n-1),\ldots,(n,0)\}$, and the corresponding Pareto optimal solutions are $0^n,10^{n-1},\ldots,1^{n}$. 
\begin{definition}[LOTZ~\cite{LaumannsTEC04}]\label{def:LOTZ}
	The LOTZ problem of size $n$ is to find $n$ bits binary strings which maximize
	\begin{align}
		{\bm{f}}(\bmx)= \left(\sum\nolimits^n_{i=1} \prod\nolimits^{i}_{j=1}x_j, \sum\nolimits^{n}_{i=1} \prod\nolimits^{n}_{j=i}(1-x_j)\right),
	\end{align}
	where $x_j$ denotes the $j$-th bit of $\bmx \in \{0,1\}^n$.
\end{definition}
We prove in Theorem~\ref{thm:bintour-lotz} that the NSGA-II can find the Pareto front in $O(n^2)$ expected number of generations, i.e., $O(n^3)$ expected number of fitness evaluations, because the generated $N$ offspring solutions need to be evaluated in each iteration. 
Note that the running time of an EA is usually measured by the number of fitness evaluations, because evaluating the fitness of a solution is often the most time-consuming step in practice.
The main proof idea can be summarized as follows.
The NSGA-II first employs the mutation operator to find the two solutions with the largest number of leading 1-bits and the largest number of trailing 0-bits,  i.e., $1^n$ and $0^n$, respectively; then employs the recommendation operator to find the whole Pareto front.  
\begin{theorem}\label{thm:bintour-lotz}
	For the NSGA-II solving LOTZ, if using binary tournament selection and a population size $N$ such that $2n+2\le N=O(n)$, then the expected number of generations for finding the Pareto front is $O(n^2)$.
\end{theorem}
\begin{proof}
	We divide the running process of NSGA-II into two phases. The first phase starts after initialization and finishes until $1^n$ and $0^n$ are both found; the second phase starts
	after the first phase and finishes when the Pareto front is found. We will show that the expected number of generations of the two phases are both $O(n^2)$, and thus prove the theorem. 
	In the following proof, we will use $\LO(\cdot)$ to denote the first objective value, i.e., the number of leading 1-bits of a solution, and $\TZ(\cdot)$ to denote the second objective value, i.e., the number of trailing 0-bits of a solution.
	
	\vspace{0.6em}
	\textbf{Analysis of the first phase.}	\vspace{0.1em}
	
	For the first phase, 
	we will prove that the expected number of generations for finding $1^n$ is $O(n^2)$, and then the same bound also holds for $0^n$ analogously.
	Let $\LOmax$ denote the maximum number of leading 1-bits of a solution in the current population $P$, i.e., 
	$\LOmax=\max\{\LO(\bmx)\mid \bmx\in P \}. $
	We first show that $\LOmax$ will not decrease during the optimization procedure of NSGA-II. Let $A=\{\bmx\in P\cup Q \mid \LO(\bmx)=\max_{\bmx\in P\cup Q} \LO(\bmx)\}$ denote the set of solutions in $P\cup Q$ with the maximum leading 1-bits, and $A^*=\{\bmx\in A\mid  \TZ(\bmx)=\max_{\bmx\in A}\TZ(\bmx)\}$ denote the set of solutions in $A$ with the maximum trailing 0-bits,
	where 
	$Q$ denotes the set of offspring solutions generated from $P$ in lines~5--9 of Algorithm~\ref{alg:nsgaii}.
	Then, the rank of any solution $\bmx\in A^*$ is 1 (i.e., $\bmx$ cannot be dominated by any other solution in $P\cup Q$), because $\bmx$ has the largest $\LO$ value and any solution with the same $\LO$ value cannot have larger $\TZ$ value. 
	Next, we consider two cases for $|F_1|$, where $F_1$ denotes the set of solutions in $P\cup Q$ with rank 1, and $|\cdot|$   denotes the size of a set. \\
	(1) $|F_1|\le N$. Then, all the solutions in $A^*$ will be included in the next population. \\
	(2) $|F_1|>N$. Then, we need to compute the crowding distances of the solutions in $F_1$ using Algorithm~\ref{alg:crowdist}, and preserve $N$ solutions with the largest crowding distance. Note that $A$ consists of the solutions with the largest $\LO$ value, and the solutions in $A\setminus A^*$ must have rank larger than 1,  thus $A^*$  actually consists of all the solutions in $F_1$ with the largest $\LO$ value. Therefore, one of the solutions in $A^*$ will be put in the last slot when sorting the solutions in $F_1$ according to the $\LO$ value, which implies such solution will have an crowding distance of $\infty$. For the bi-objective problem LOTZ, we need to sort the solutions in $F_1$ twice, i.e.,  according to the $\LO$ value and the $\TZ$ value, respectively, 
	implying that there are at most four solutions whose crowding distance can be assigned to $\infty$. Consequently, at least one solution in $A^*$ belongs to the four best solutions in $F_1$ (w.r.t. $\succ_{\mathrm{c}}$), and thus can be included in the next population. \\
	Combining the two cases, we have shown that there exists one solution in the next population whose number of leading 1-bits is $\max_{\bmx\in P\cup Q}\LO(\bmx)$, which is obviously not smaller than $\LOmax$.	 
	
	Next, we show that $\LOmax$ ($<n$) can increase by at least 1 with probability at least $\Omega(1/n)$ in each generation.
	Similar to the analysis in the above paragraph, there exists one solution $\bmx^*\in \{\bmx\in P\mid \LO(\bmx)=\LOmax\} $ such that   $\rank{\bmx^*}=1$ and $\dist{\bmx^*}=\infty$. 
	Recall that when using the tournament selection to select a parent solution, the competition between the two randomly selected solutions is based on their ranks and crowding distances (in case of equal ranks).
	Thus, once $\bmx^*$ is selected for competition (whose probability is $1/N$), it will always win, if the other solution selected for competition has larger rank or finite crowding distance; or win with probability $1/2$, if the other solution has the same rank and crowding distance as $\bmx^*$, resulting in a tie which is broken uniformly at random.\\
	Suppose $\bmx^*$ becomes a parent solution, then it will generate two offspring solutions together with another parent solution by crossover and mutation.
	In the reproduction procedure, $\bmx^*$ can keep unchanged after crossover with probability at least 0.1 by line~6 of Algorithm~\ref{alg:nsgaii}, and flip only its $(\LO(\bmx^*)+1)$-th bit (which is a 0-bit) with probability $(1/n)\cdot (1-1/n)^{n-1}\ge 1/(en)$ by bit-wise mutation. Thus, the probability of generating an offspring solution $\bmy$ with $\LO(\bmy)\ge \LO(\bmx^*)+1= \LOmax+1$ is at least $0.1\cdot 1/(en)$. \\
	In each generation, $N$ (i.e., $N/2$ pairs of) parent solutions will be selected and produce $N/2$ pairs of offspring solutions, thus the probability of generating a solution with more than $\LOmax$ leading 1-bits is at least 
	\begin{equation}\label{eq:LOmax-onestep}
		\begin{aligned}
			&1-\Big(1-\frac{1}{2N}\cdot 0.1\cdot \frac{1}{en}\Big)^{N/2}
			\ge 1-e^{-(1/(20enN))\cdot (N/2)}\\
			& = 1-\frac{1}{e^{1/(40en)}}\ge  1-\frac{1}{1+1/(40en)}= \frac{1}{40en+1}=\Omega\Big(\frac{1}{n}\Big),
		\end{aligned}
	\end{equation}
	where the inequalities hold by $1+a\le e^a$ for any $a\in \mathbb{R}$. 
	By the analysis in the previous paragraph, there must exist a solution $\bmy^*\in P\cup Q$ with the largest number of leading 1-bits such that $\rank{\bmy^*}=1$ and $\dist{\bmy^*}=\infty$, and thus will be maintained in the next population.
	Hence, once an offspring solution with the number of leading 1-bits larger than $\LOmax$ is generated, $\LOmax$ will increase by at least 1 in the next population.
	
	Note that the initial value of $\LOmax$ is at least 0, thus the expected number of generations for increasing $\LOmax$ to $n$, i.e., finding $1^n$, is at most $O(n^2)$. Analogously, we can derive that the expected number of generations for finding $0^n$ is also $O(n^2)$. 
	
	\vspace{0.6em}
	\textbf{Analysis of the second phase.}	\vspace{0.1em}
	
	Now we consider the second phase, i.e., finding the whole Pareto front. We first show that once an objective vector $\bmf^*$ in the Pareto front has been found, it will always be maintained in the population. To this end, we first show that for any $i\in \{0,1,\ldots,n\}$, there exist at most two solutions in $P\cup Q$ with $i$ leading 1-bits, such that their ranks are equal to 1 and crowding distances are larger than 0.  
	Given any $i\in \{0,1,\ldots,n\}$, we simply assume that there exists at least one solution in $P\cup Q$ with $i$ leading 1-bits, because otherwise,  the claim already holds. Let 
	 \begin{equation}\label{eq:bintour-lotz-Bi}
	 	\begin{aligned}
	 		B^i= \{\bmx\in P\cup Q\mid \LO(\bmx)=i \wedge 
	 		\TZ(\bmx)=\max_{\bmx'\in P\cup Q, \LO(\bmx')=i}\TZ(\bmx')\}
	 	\end{aligned}
	 \end{equation} 
 	denote the set of solutions which have $i$ leading 1-bits and meanwhile have the maximum number of trailing 0-bits. Then, for any $\bmx\in (P\cup Q)\setminus B^i$ satisfying $\LO(\bmx)=i$, it must hold $\rank{\bmx}>1$, because $\bmx$ can be dominated by any solution in $B^i$.   Thus, we only need to consider the solutions in $B^i$.
	If $|B^i|\le 2$, where $|\cdot|$ denotes the cardinality of a set, then the claim trivially holds.
	If $|B^i|\ge 3$,  then at most two solutions in $B^i$ can have crowding distances larger than 0, because all the solutions in $B^i$ have the same objective vector, and one solution can be assigned a crowding distance larger than 0 only if it is put in the first or the last position among the solutions in $B^i$ when the solutions are sorted according to some objective function.	
	Note that here we use the assumption in Section~\ref{sec:nsgaii}, i.e., when the solutions  with the same objective vector are sorted according to some objective function $f_j$, theirs positions are unchanged or totally reversed.\\
	Now we show that there exists at least one solution $\bmx$ corresponding to $\bmf^*$ such that $\rank{\bmx}=1$ and $\dist{\bmx}>0$, and then conclude the statement, i.e., $\bmf^*$ will always be maintained in the population. Let $C$ denote the set of solutions in $P\cup Q$ whose objective vectors are identical to $\bmf^*$. Then, any solution in $C$ has rank 1, because it cannot be dominated by any other solution. When the solutions in $C$ are sorted according to some objective function, one solution (denoted as $\hat{\bmx}$) will be put in the first or the last position, thus having a crowding distance larger than 0 by line~6 of Algorithm~\ref{alg:crowdist}. Thus, $\hat{\bmx}$ will not be inferior to $(2n+2)$ solutions in $P\cup Q$, because otherwise there must exist other $(2n+2)$ solutions with rank 1 and crowding distance larger than 0, which leads to a contradiction. Note that the population size $N\ge 2n+2$, we can derive that $\hat{\bmx}$ will be kept in the next population, which implies $\bmf^*$ will also be maintained.
	
	Next, we consider the expansion of the Pareto fount. 
	We first analyze the probability of selecting $1^n$ or $0^n$ as a parent solution.
	Note that at least one $1^n$ in $P$ has rank 1 and crowding distance $\infty$, thus once it is selected for competition, whose probability is $1/N$, it will win the competition with probability at least $1/2$, where $1/2$ is the probability of breaking a tie. Hence, the probability of selecting $1^n$ as a parent solution is least $1/(2N)$, and the same bound also holds for $0^n$ by a similar argument.
	
	We now analyze the probability of generating a new Pareto optimal solution. 
	Let 
	\begin{equation}\label{eq:bintoru-lotz-D}
		D=\{(j,n-j)\mid\bmx \in P \wedge \bmx=1^j0^{n-j},  j \in\{1,2,\ldots,n-1\} \}
	\end{equation}
	 denote the set of objective vectors of the Pareto optimal solutions in $P$ (except $1^n$ and $0^n$). 
	Suppose currently the size of $D$ is equal to $i$.  If $i=0$, then selecting $1^n$ and $0^n$ as a pair of parent solution, and exchanging the first $j$ bits ($1\le j\le n-1$) can generate a new Pareto optimal solution $1^j0^{n-j}$. The probability of such event is at least 
	\begin{equation}\label{eq:second-path-i0}
		\begin{aligned}
			\frac{2}{(2N)^2}\cdot 0.9\cdot \frac{n-1}{n} \cdot \big(1-\frac{1}{n}\big)^n 
			\ge \frac{0.9}{2eN^2}\cdot \big(1-\frac{1}{n}\big)^2\ge \frac{1}{4eN^2},
		\end{aligned}
	\end{equation}
	where the term $2/(2N)^2$ denotes the probability of selecting $1^n$ and $0^n$, or $0^n$ and $1^n$ as a pair of solutions, the term 0.9 denotes the probability of performing the crossover operator, the term $(n-1)/n$ denotes the probability of selecting one of the $(n-1)$ crossover points, the term $(1-1/n)^n$ denotes the probability of not flipping any bits by mutation, the first inequality holds by $(1-1/n)^{n-1}\ge 1/e$, and the last inequality holds for $n\ge 4$.  \\
	Then, we consider the case $i>0$. 
	In one binary tournament selection procedure, the probability of selecting two solutions with objective vectors in $D$ is at least $i^2/N^2$, and we denote the winning solution, i.e., the parent solution, as $\tilde{\bmx}$.
	Suppose the number of leading 1-bits of $\tilde{\bmx}$ is $k$, 
	and let $D_1=\{j\mid (j,n-j)\in D\wedge j<k\}$,
	$D_2=\{j\mid (j,n-j)\in D\wedge j>k\}$.
	If the other parent solution is $0^n$, then exchanging the first $k_1$ ($1\le k_1\le k-1,k_1\notin D_1$) bits of $\tilde{\bmx}$ and $0^n$ can generate a new Pareto optimal solution $1^{k_1}0^{n-k_1}$, whose probability is $(k-1-|D_1|)/n$; if the other parent solution is $1^n$, then exchanging the first $k_2$ ($k+1\le k_2\le n-1,k_2\notin D_2$) bits of $\tilde{\bmx}$ and $1^n$ can generate a new Pareto optimal solution $1^{k_2}0^{n-k_2}$, whose probability is $(n-1-k-|D_2|)/n$. Note that the probability of selecting $1^n$ (or $0^n$) as a parent solution is lower bounded by $1/(2N)$, thus we can derive that the probability of generating  a new Pareto optimal solution in $P$ is at least
	\begin{equation}
		\begin{aligned}\label{eq:second-path-i>0}
			&\frac{i^2}{N^2}\cdot  \frac{1}{2N}\cdot 0.9\cdot \Big(\frac{k-1-|D_1|}{n}+\frac{n-1-k-|D_2|}{n}\Big)\cdot \big(1-\frac{1}{n}\big)^n \\
			&\ge \frac{i^2}{N^2}\cdot\bigg( \frac{1}{2N}\cdot \frac{n-2-|D_1|-|D_2|}{n}\cdot \frac{1}{2e}\bigg) 
			=\frac{i^2(n-1-i)}{4enN^3},
		\end{aligned}
	\end{equation}
	where the term $0.9$ denotes the probability of using the crossover operator, 
	the term $(1-1/n)^n$ denotes the probability of not flipping any bits by mutation, and the inequality holds by $(1-1/n)^{n-1}\ge 1/e$ and $0.9\cdot (1-1/n)>1/2$ for $n\ge 3$.\\
	Note that NSGA-II performs $N/2$  reproduction procedures, i.e., selection, crossover and mutation, in each generation. Thus, by Eqs.~\eqref{eq:second-path-i0} and~\eqref{eq:second-path-i>0}, the probability of generating a new objective vector in Pareto front is at least 
	\begin{equation}
		\begin{aligned}
			1-\Big(1-\frac{1}{4eN^2}\Big)^{N/2}\ge 1-e^{-1/(8eN)}
			\ge 1-\frac{1}{1+1/(8eN)}=\frac{1}{8eN+1}
		\end{aligned}
	\end{equation}
	for $i=0$, and
	\begin{equation}
		\begin{aligned}
			&1-\Big(1-\frac{i^2(n-1-i)}{4enN^3}\Big)^{N/2}
			\ge 1-e^{-i^2(n-1-i)/(8enN^2)}\\
			&
			\ge 1-\frac{1}{1+i^2(n-1-i)/(8enN^2)}
			\ge \frac{i^2(n-1-i)}{8enN^2+i^2(n-1-i)}
		\end{aligned}
	\end{equation}
	for $i>0$, where the inequalities hold by $1+a\le e^a$ for any $a\in \mathbb{R}$.
	
	Now, we can derive that the expected number of generations for finding the Pareto front is at most 
	\begin{equation}\label{eq:second-path-sum}
		\begin{aligned}
			&8eN\!+\!1\!+\sum_{i=1}^{n-2}\Big(\frac{8enN^2}{i^2(n-1-i)}+1\Big)
			=O(n)+8enN^2\sum_{i=1}^{n-2}\frac{1}{i^2(n-1-i)},
		\end{aligned}
	\end{equation}
	where the equality is by $N=O(n)$. Note that 
	\begin{align}
		\frac{1}{i^2(n-1-i)}
		=\frac{1}{n-1}\cdot \frac{1}{i^2}+\frac{1}{(n-1)^2}\cdot \Big(\frac{1}{i}+\frac{1}{n-1-i}\Big),
	\end{align}
	thus Eq.~\eqref{eq:second-path-sum} continues with 
	\begin{align}
		&=O(n)+8enN^2\sum_{i=1}^{n-2}\bigg(\frac{1}{n-1}\cdot \frac{1}{i^2}+\frac{1}{(n-1)^2}\cdot \Big(\frac{1}{i}+\frac{1}{n-1-i}\Big)\bigg)\\
		&=O(n)+\frac{8enN^2}{n-1}\bigg(1+\sum_{i=2}^{n-2}\frac{1}{i^2}+\frac{1}{n-1}\sum_{i=1}^{n-2}\Big(\frac{1}{i}+\frac{1}{n-1-i}\Big)\bigg)\\
		&\le O(n)+\frac{8enN^2}{n-1}\bigg(1+\sum_{i=2}^{n-2}\Big(\frac{1}{i-1}-\frac{1}{i}\Big)+\frac{2(1+\ln (n-2))}{n-1}\bigg)\\
		&\le O(n)+\frac{8enN^2}{n-1}\bigg(1+1-\frac{1}{n-2}+o(1)\bigg)
		=O(n^2),
	\end{align}
	where the first inequality is by $\sum_{i=1}^{j}1/i\le 1+\ln j$ for any $j\ge 1$. Hence, we have shown an upper bound $O(n^2)$ for the expected number of generations of the second phase.
\end{proof}

The OneMinMax problem presented in Definition~\ref{def:OMM} aims to maximize the number of 0-bits and the number of 1-bits of a binary bit string.  
The Pareto front of OneMinMax is $\mathcal{F}=\{(0,n),(1,n-1),\ldots,(n,0)\}$, and 
any solution $\bmx\in\{0,1\}^n$ is Pareto optimal, corresponding to the objective vector $(n-|\bmx|_1,|\bmx|_1)$ in the Pareto front, where $|\cdot|_1$  denotes the number of 1-bits of a solution.
\begin{definition}[OneMinMax~\cite{Giel10}]\label{def:OMM}
	The OneMinMax problem of size $n$ is to find $n$ bits binary strings which maximize
	\begin{align}
		{\bm f}(\bmx)=\left(n-\sum\nolimits^n_{i=1}x_i, \sum\nolimits^{n}_{i=1} x_i\right),
	\end{align}
	where $x_i$ denotes the $i$-th bit of $\bmx \in \{0,1\}^n$.
\end{definition}
We prove in Theorem~\ref{thm:bintour-omm} that the NSGA-II can find the Pareto front in $O(n\log n)$ expected number of generations, i.e., $O(n^2\log n)$ expected running time.
\begin{theorem}\label{thm:bintour-omm}
	For the NSGA-II solving OneMinMax, if using binary tournament selection and a population size $N$ such that $2n+2\le N= O(n)$, then the expected number of generations for finding the Pareto front is $O(n\log n)$.
\end{theorem}
\begin{proof}
	The proof is similar to that of Theorem~\ref{thm:bintour-lotz}, i.e., we divide the optimization procedure into two phases, where the first phase starts after initialization and finishes until $1^n$ and $0^n$ are both found; the second phase starts	after the first phase and finishes when the Pareto front is found. However, for OneMinMax, we will show that the expected numbers of generations of the two phases are both $O(n\log n)$, instead of $O(n^2)$. 
	
	\textbf{For the first phase}, we need to consider the increment of a quantity $\Omax$, which is defined as $\Omax=\max\{|\bmx|_1\mid \bmx\in P\}$, where $|\cdot |_1$ denotes the number of 1-bits of a solution.  Then, similar to the analysis of $\LOmax$, $\Omax$ will not decrease, and we need to analyze the probability that $\Omax$ increases by at least 1 in each generation. Let $\bmx^*$ be a solution in $\{\bmx\in P\mid |\bmx|_1=\Omax\} $ such that   $\rank{\bmx^*}=1$ and $\dist{\bmx^*}=\infty$, then $\bmx^*$  will be selected for competition in the binary tournament selection with probability at least $1/N$, and can win with probability at least 1/2.  In the reproduction procedure, $\bmx^*$  can generate an offspring solution  $\bmy$ such that $|\bmy|_1\ge |\bmx^*|_1+1$ with  probability at least $0.1\cdot ((n-|\bmx^*|_1)/n)\cdot (1-1/n)^{n-1}\ge  (n-|\bmx^*|_1)/(10en)$,  where the term $0.1$ denotes the probability of not using recommendation operator, and the term $((n-|\bmx^*|_1)/n)\cdot (1-1/n)^{n-1}$ denotes the probability of flipping one 0-bit of $\bmx^*$. Then, similar to Eq.~\eqref{eq:LOmax-onestep}, the probability of generating a solution with more than $|\bmx^*|_1$ 1-bits in each generation is at least 
	\begin{equation}\label{eq:Omax-onestep}
		\begin{aligned}
			&1-\Big(1-\frac{1}{2N}\cdot \frac{n-|\bmx^*|_1}{10en}\Big)^{N/2}
			\ge 1-\frac{1}{e^{(n-|\bmx^*|_1)/(40en)}}\\
			&\ge  1-\frac{1}{1+(n-|\bmx^*|_1)/(40en)}= \frac{n-|\bmx^*|_1}{40en+(n-|\bmx^*|_1)}\\
			&=\Omega\Big(\frac{n-|\bmx^*|_1}{n}\Big).
		\end{aligned}
	\end{equation}
	Thus,  the expected number of generations for finding $1^n$ is at most 
	\begin{equation}\label{eq:Omax-time}
		\sum_{i=0}^{n-1}O\Big(\frac{n}{n-i}\Big)=O(n\log n),
	\end{equation}
	and the bound also holds for $0^n$ by a similar analysis procedure. 
	
	\textbf{Now we consider the second phase}.	Similar to Eq.~\eqref{eq:bintour-lotz-Bi},  we define a set $B^i=\{\bmx\in P\cup Q \mid |\bmx|_1=i\}$ ($0\le i\le i$), which denotes the set of solutions which have $i$ 1-bits. Note that we do not add any restriction to the other objective value, i.e., the number of 0-bits, because the number of 0-bits of a solution can be decided by the number of 1-bits. Then, following the analysis in the proof of Theorem~\ref{thm:bintour-lotz}, we can show that there exist at most two solutions in $P\cup Q$ with $i$ 1-bits such that their ranks are equal to 1 and crowding distances are larger than 0, and thus prove that an objective vector $\bmf^*$ in the Pareto front will always be maintained in the population once it has been found. \\
	Now we examine the probability of generating a new Pareto optimal solution in each generation. Similar to Eq.~\eqref{eq:bintoru-lotz-D}, let $D=\{(n-|\bmx|_1,|\bmx|_1)\mid\bmx \in P \wedge 1\le |\bmx|_1\le n-1 \}$ denote the set of objective vectors of the solutions in $P$ (except $1^n$ and $0^n$), and suppose currently the size of $D$ is equal to $i$. 
	Note that for OneMinMax, any solution $\bmx\in\{0,1\}^n$ is Pareto optimal, thus the following analysis is a little easier. First, we consider the case that the number of $1^n$ or the number of $0^n$ is larger than $N/4$.  Without loss of generality, we assume that the number of $0^n$ in $P$ is  at least $N/4$. Then,  the probability of selecting $0^n$ as a parent solution is at least $(1/4)^2=1/16$,  because we only need to select $0^n$ twice in binary tournament selection. By the analysis in Theorem~\ref{thm:bintour-lotz}, the probability of selecting $1^n$ as the other parent solution is at least $1/(2N)$. Then, exchanging the first $j$ ($1\le j\le n-1 \wedge (n-j,j)\notin D$) bits of $0^n$ and $1^n$ can generate a new Pareto optimal solution $1^j0^{n-j}$, whose probability  is at least $(n-1-i)/n$. Thus,  combining the above-mentioned probabilities, we can derive that the probability of generating a new Pareto
	optimal solution not in $P$ is at least $\Omega((n-1-i)/(nN))$.\\
	Then, we consider the case that the number of $1^n$ and the number of $0^n$ in $P$ are both smaller than or equal to $N/4$. In one binary tournament selection procedure, the probability of selecting two solutions with objective vectors in $D$ is at least $(N-N/4-N/4)^2/N^2=1/4$,  because it is sufficient to not select $0^n$ or $1^n$. Let $\tilde{\bmx}$ denote the winning solution, i.e., the parent solution, and suppose $|\tilde{\bmx}|_1=k$. 
	If the other parent solution is $0^n$, then for any $1\le k_1\le k-1,(n-k_1,k_1)\notin D$,  there must exist a crossover point $k_1'$ such that exchanging the first $k_1'$ bits of $\bmx$ and $0^n$ can generate a Pareto optimal solution with $k_1$ 1-bits. 
	If the other parent solution is $1^n$, then for any $k+1\le k_2\le n-1,(n-k_2,k_2)\notin D$,  there must exist a crossover point $k_2'$ such that exchanging the first $k_2'$ bits of $\bmx$ and $1^n$ can generate a Pareto optimal solution with $k_2$ 1-bits.
	Note that the probability of selecting $1^n$ (or $0^n$) as a parent solution is at least $1/(2N)$, thus similar to Eq.~\eqref{eq:second-path-i>0}, we can derive that the probability of generating  a new Pareto optimal solution not in $P$ is at least
	\begin{equation}
		\begin{aligned}
			\frac{1}{4}\cdot \frac{1}{2N}\cdot 0.9\cdot \frac{n-1-i}{n}\cdot \big(1-\frac{1}{n}\big)^n \ge \frac{n-1-i}{16enN}=\Omega\Big(\frac{n-1-i}{nN}\Big).
		\end{aligned}
	\end{equation}
	Thus, in each generation, the probability of generating a new objective vector in Pareto front is at least 
	\begin{equation}\label{eq:bintour-omm-phase2-prob}
		\begin{aligned}
			&1-\Big(1-\Omega\Big(\frac{n-1-i}{nN}\Big)\Big)^{N/2}\ge 1-e^{-\Omega((n-1-i)/(2n))}\\
			&\ge 1-\frac{1}{1+\Omega((n-1-i)/(2n))},
		\end{aligned}
	\end{equation}
	Then, we can derive that the expected number of generations for finding the whole Pareto front is at most 
	\begin{equation}\label{eq:bintour-omm-phase2-time}
		\begin{aligned}
			&\sum_{i=0}^{n-2}\Big(1-\frac{1}{1+\Omega((n-1-i)/(2n))}\Big)^{-1}\\
			&= 				\sum_{i=0}^{n-2}\Big(1+\frac{1}{\Omega((n-1-i)/(2n))}\Big)\\
			&=n-1+\sum_{i=0}^{n-2}O\Big(\frac{2n}{n-1-i}\Big)=
			O(n\log n),
		\end{aligned}
	\end{equation}
	where the last equality is by $\sum_{i=1}^{j}1/i\le 1+\ln j$ for any $j\ge 1$.  Thus, combining the analyses for the two phases, the Theorem holds.
\end{proof}

The COCZ problem as presented in Definition~\ref{def:COCZ} is similar to OneMinMax, but is a little complicated. Its first objective is to maximize the number of 1-bits of a solution, and the other objective is to maximize the number of 1-bits in the first half of the solution plus the number of 0-bits in the second half.  That is, the two objectives are consistent in maximizing the number of 1-bits in the first half of the solution, but conflict in the second half.
The Pareto front of COCZ is $\mathcal{F}=\{(n/2,n),(n/2+1,n-1),\ldots,(n,n/2)\}$, and any solution $\bmx$ satisfying $\sum_{i=1}^{n/2}x_i=n/2$
is Pareto optimal, corresponding to the objective vector $(n/2+\sum_{i=n/2+1}^{n}x_i,n/2+\sum_{i=n/2+1}^{n}(1-x_i))$ in the Pareto front. 
\begin{definition}[COCZ~\cite{LaumannsTEC04}]\label{def:COCZ}
	The COCZ problem of size $n$ is to find $n$ bits binary strings which maximize
	\begin{align}
		{\bm f}(\bmx)=\left(\sum\nolimits^n_{i=1} x_i, \sum\nolimits^{n/2}_{i=1} x_i +\sum\nolimits^{n}_{i=n/2+1} (1-x_i)\right),
	\end{align}
	where $n$ is even and $x_i$ denotes the $i$-th bit of $\bmx \in \{0,1\}^n$.
\end{definition}
We prove in Theorem~\ref{thm:bintour-cocz} that the NSGA-II can find the Pareto front in $O(n\log n)$ expected number of generations, i.e., $O(n^2\log n)$ expected running time.
\begin{theorem}\label{thm:bintour-cocz}
	For the NSGA-II solving COCZ, if using binary tournament selection and a population size $N$ such that $n+2\le N= O(n)$, then the expected number of generations for finding the Pareto front is $O(n\log n)$.
\end{theorem}
\begin{proof}
	The proof is similar to that of Theorems~\ref{thm:bintour-lotz} and~\ref{thm:bintour-omm}, i.e., we will divide the optimization procedure into two phases. However, the target of the first phase is a little different, i.e., we need to find $1^n$ and $1^{n/2}0^{n/2}$,  instead of $1^n$ and $0^n$. We will show that the expected numbers of generations of the two phases are both $O(n\log n)$, i.e., the same as that of OneMinMax.
	In the following discussion, we will use $\CZ(\bmx)$ to denote the second objective value of a solution $\bmx$, i.e., $\CZ(\bmx)=\sum^{n/2}_{i=1} x_i +\sum^{n}_{i=n/2+1} (1-x_i)$.
	
	\textbf{For the first phase}, the analysis for $1^n$ is almost the same as that of Theorem~\ref{thm:bintour-omm}, and we mainly examine the expected number of generations for finding $1^{n/2}0^{n/2}$. Let $\CZmax=\max\{\CZ(\bmx) \mid \bmx\in P\}$ denote the maximum second objective value of a solution in $P$, then similar to the analysis in the previous theorems, $\CZmax$ will not decrease, and we need to consider the increase of $\CZmax$. Let $\bmx^*$ be a solution in $\{\bmx\in P\mid \CZ(\bmx)=\CZmax\} $ such that   $\rank{\bmx^*}=1$ and $\dist{\bmx^*}=\infty$, then following the analysis in the Theorem~\ref{thm:bintour-omm}, $\bmx^*$ can be selected as a parent solution with probability at least $1/(2N)$. Suppose  the number of 1-bits in the first half of $\bmx^*$ is $k_1$, i.e., $\sum^{n/2}_{i=1} x_i^*=k_1$, and the number of 0-bits in the second half of $\bmx^*$ is $k_2$, i.e., $\sum^{n/2}_{i=1}(1- x_i^*)=k_2$.  Then, flipping one of the $(n/2-k_1)$ 0-bits in the first half of $\bmx^*$, or flipping one of the $(n/2-k_2)$ 1-bits in the second half of $\bmx^*$ can  generate an offspring solution $\bmy$ such that $\CZ(\bmy)\ge \CZ(\bmx^*)+1=\CZmax+1$. 
	Then, similar to Eq.~\eqref{eq:Omax-onestep}, the probability of generating a solution with the $\CZ$ value larger than $\CZmax$ in each generation is at least 
	\begin{equation}
		\begin{aligned}
			&1-\Big(1-\frac{1}{2N}\cdot \frac{n/2-k_1+n/2-k_2}{10en}\Big)^{N/2}
			\\
			&=1-\Big(1-\frac{1}{2N}\cdot \frac{n-\CZmax}{10en}\Big)^{N/2}
			=\Omega\Big(\frac{n-\CZmax}{n}\Big),
		\end{aligned}
	\end{equation}
	where the first equality is by $k_1+k_2=\CZ(\bmx^*)=\CZmax$, and the last equality is by the same derivation in Eq.~\eqref{eq:Omax-onestep}.
	Thus,  similar to Eq.~\eqref{eq:Omax-time}, the expected number of generations for increasing $\CZmax$ to $n$, i.e., finding $1^{n/2}0^{n/2}$, is at most $O(n\log n)$.
	
	\textbf{For the second phase},  we first need to show that  once an objective vector $\bmf^*$ in the Pareto front is found, it will always be maintained. Similar to Eq.~\eqref{eq:bintour-lotz-Bi}, let 
	\begin{align}
		B^i=\Big\{\bmx\in P\cup Q\mid &\sum\nolimits^n_{i=n/2+1} x_i = i \wedge\\ &\sum\nolimits^{n/2}_{i=1} x_i=\max_{\bmx'\in P\cup Q, \sum\nolimits^n_{i=n/2+1} x_i'=i}  \sum\nolimits^{n/2}_{i=1} x_i'\Big\}
	\end{align}
	denote the set of solutions which have $i$ 1-bits in the second half and meanwhile have the maximum number of 1-bits in the first half.  Then, following the analysis after Eq.~\eqref{eq:bintour-lotz-Bi}, we can prove that for any $i\in \{0,1,\ldots,n/2\}$, there exist at most two solutions in $P\cup Q$ with the numbers of the 1-bits in the second half equal to $i$, such that their ranks are equal to 1 and crowding distances are larger than 0. Then,  similar to the statement in Theorem~\ref{thm:bintour-lotz}, we can prove the claim, i.e., an objective vector $\bmf^*$ in the Pareto front will always be maintained once it has been found.
	
	Now we consider the expansion of the Pareto front. Let  $D=\{(|\bmx|_1,\CZ(\bmx))\mid\bmx \in P \wedge n/2<|\bmx|_1<n \wedge \sum\nolimits^{n/2}_{i=1} x_i = n/2\} $ denote the set of objective vectors of the solutions in $P$ (except $1^n$ and $1^{n/2}0^{n/2}$), and suppose currently the size of $D$ is equal to $i$. 
	Note that in one reproduction procedure, we need to select two parent solutions. Suppose we are given a parent solution $\tilde{\bmx}$ such that $\sum\nolimits^{n}_{i=n/2+1} \tilde{x}_i =k$, we will show that the probability of generating a new Pareto optimal solution not in $P$ is at least $\Omega((n/2-1-i)/(nN))$. 
	If $k=0$, then selecting $1^n$ as the other parent solution and exchanging the first $k'$ (where $n/2<k'<n,(k',n/2+n-k')\notin D$) bits of two parent solutions can generate a new Pareto optimal solution $1^{k'}0^{n-k'}$; if $k=n/2$, then selecting $1^{n/2}0^{n/2}$ as the other parent solution and exchanging the first $k'$ (where $n/2<k'<n,(n/2+n-k',n/2+k'-n/2)\notin D$) bits of two parent solutions can generate a new Pareto optimal solution $1^{n/2}0^{k'-n/2}1^{n-k'}$. Note that the probability of selecting $1^n$ (or $0^n$) as a parent solution is at least $1/(2N)$, thus in both cases, the probability of generating a new Pareto optimal solution not in $P$ is at least $\Omega((n/2-1-i)/(nN))$. If $0<k<n/2$, then similar to the analysis of the second phase in Theorems~~\ref{thm:bintour-lotz} and~\ref{thm:bintour-omm}, we can also derive that the probability of generating  a new Pareto optimal solution not in $P$ is at least $\Omega((n/2-1-i)/(nN))$.\\
	Thus, similar to Eqs.~\eqref{eq:bintour-omm-phase2-prob} and~\eqref{eq:bintour-omm-phase2-time}, the expected number of generations for finding the whole Pareto front is at most $O(n\log n)$.  		 
\end{proof}

\section{Analysis of NSGA-II Using Stochastic Tournament Selection}
In the previous section, we have proved that the expected running time of the standard NSGA-II is $O(n^3)$ for LOTZ, and $O(n^2\log n)$ for OneMinMax and COCZ, which is as same as that of the previously analyzed simple MOEAs, GSEMO and SEMO~\cite{LaumannsTEC04,Giel03,Qian13,Giel10}.
Next, we will show that by employing stochastic tournament selection in Definition~\ref{def:sto-tour} instead of binary tournament selection, the NSGA-II needs much less time to find the whole Pareto front. 

In particular, we prove that the expected number of generations of the NSGA-II using stochastic tournament selection is $O(n)$ (implying $O(n^2)$ expected running time) for solving all the three problems, in Theorems~\ref{thm:stotour-lotz}--\ref{thm:stotour-cocz}.
The working principle of the NSGA-II observed in the proofs of these theorems is similar to that observed in the previous section. That is, the NSGA-II first employs the mutation operator to find the solutions that maximize each objective function,  and then employs the crossover operator to quickly find the remaining objective vectors in the Pareto front.
However, the utilization of stochastic tournament selection can make the NSGA-II select prominent  solutions, i.e., solutions maximizing each objective function, with larger probability, 
making the crossover operator easier fill in the remaining Pareto front and thus reducing the total running time.
\begin{theorem}\label{thm:stotour-lotz}
	For the NSGA-II solving LOTZ, if using stochastic tournament selection and a population size $N$ such that $2n+2\le N= O(n)$, then the expected number of generations for finding the Pareto front is $O(n)$.
\end{theorem}
\begin{proof}
	The proof is similar to that of Theorem~\ref{thm:bintour-lotz}. The main difference is the probability of selecting a specific solution from $P$, which will affect the running time complexity of both the first phase and the second phase.  
	
	\textbf{For the first phase},  the probability of selecting the solution $\bmx^*$ is $q=\Omega(1)$ by Lemma~\ref{lem:stotour-prob},  instead of $\Omega(1/N)$.  Then, similar to Eq.~\eqref{eq:LOmax-onestep}, the probability of generating a solution with more than $\LOmax$  leading 1-bits  in each generation is at least 
	\begin{equation}\label{eq:LOmax-onestep-stotour}
		\begin{aligned}
			&1-\Big(1-\frac{q}{2}\cdot \frac{1}{en}\Big)^{N/2} 
			\ge 1-\frac{1}{e^{qN/(4en)}}\\
			&\ge  1-\frac{1}{1+qN/(4en)} \ge 1-\frac{1}{1+q/(2e)}= \frac{q}{2e+q}=\Omega(1),
		\end{aligned}
	\end{equation}
	where the last inequality is by $N\ge 2n+2$. Thus,  we can derive an upper bound $O(n)$ on
	the expected number of generations for finding $1^n$, and the bound also holds for $0^n$ similarly. 
	
	\textbf{For the second phase}, we consider the  case that the parent solutions are exactly $0^n$ and $1^n$, instead of the case in the proof of Theorem~\ref{thm:bintour-lotz}, i.e., one parent solution is selected from the set of Pareto optimal solutions in $P$, and the other parent solution is selected from $0^n$ or $1^n$. 
	By Lemma~\ref{lem:stotour-prob},  the probability of selecting $0^n$ (or $1^n$) as a parent solution is $\Omega(1)$,  thus the probability that $0^n$ and $1^n$ are selected as a pair of parent solutions is also  $\Omega(1)$.  Suppose currently $i$ objective vectors in the Pareto front (except $(n,0)$ and $(0,n)$) have been found.  Note that exchanging the first $j$-th ($1\le j\le n-1$) bits of $1^n$ and $0^n$ can generate a Pareto optimal solution $1^j0^{n-j}$,  thus the probability of generating a new objective vector in the Pareto front is at least
	\begin{equation}\label{eq:second-path-stotour}
		\begin{aligned}
			\Omega(1)\cdot 0.9\cdot \frac{n-1-i}{n}\cdot \Big(1-\frac{1}{n}\Big)^n = \Omega\Big(\frac{n-1-i}{n}\Big),
		\end{aligned}
	\end{equation}
	where the term $0.9$ denotes the probability of using the crossover operator, 
	the term $(n-1-i)/n$ denotes the probabilities of selecting one crossover point, and the term $(1-1/n)^n$ denotes the probability of not flipping any bits by mutation.
	In each generation, NSGA-II produces $N/2$ pairs of offspring solutions, thus, the probability of generating a new objective vector in Pareto front is at least 
	\begin{equation}
		\begin{aligned}
			&1-\Big(1-\Omega\Big(\frac{n-1-i}{n}\Big)\Big)^{N/2}\ge 1-e^{-\Omega(N(n-1-i)/(2n))}\\
			&\ge 1-e^{-\Omega(n-1-i)}=1-\frac{1}{e^{\Omega(n-1-i)}}\ge 1-\frac{1}{1+\Omega(n-1-i)},
		\end{aligned}
	\end{equation}
	where the first and third inequalities hold by $1+a\le e^a$ for any $a\in \mathbb{R}$,  and the second inequality holds by $N\ge 2n+2$. Then, we can derive that the expected number of generations for finding the whole Pareto front is at most 
	\begin{equation}
		\begin{aligned}
			&\sum_{i=0}^{n-2}\frac{1}{1-1/(1+\Omega(n-1-i))}= 				\sum_{i=0}^{n-2}\Big(1+\frac{1}{\Omega(n-1-i)}\Big)\\
			&=n-1+O(1+\ln (n-1))=O(n),
		\end{aligned}
	\end{equation}
	where the first inequality is by $\sum_{i=1}^{j}1/i\le 1+\ln j$ for any $j\ge 1$.  Thus,  combining the analyses for the two phases leads to the theorem.
\end{proof}
 The proofs of Theorems~\ref{thm:stotour-omm} and~\ref{thm:stotour-cocz} are omitted, because they are almost as same as that of Theorem~\ref{thm:stotour-lotz}. We only need to incorporate the properties of OneMinMax and COCZ revealed in the proofs of  Theorems~\ref{thm:bintour-omm} and~\ref{thm:bintour-cocz}.
That is, an objective vector in the Pareto front will always be maintained once it has been found, if the population size $N$ is at least $2n+2$ for OneMinMax and at least $n+2$ for COCZ.
\begin{theorem}\label{thm:stotour-omm}
	For the NSGA-II solving OneMinMax, if using stochastic tournament selection and a population size $N$ such that $2n+2\le N= O(n)$, then the expected number of generations for finding the Pareto front is $O(n)$.
\end{theorem}
\begin{theorem}\label{thm:stotour-cocz}
	For the NSGA-II solving COCZ, if using stochastic tournament selection and a population size $N$ such that $n+2\le N= O(n)$, then the expected number of generations for finding the Pareto front is $O(n)$.
\end{theorem}

\section{Experiments}

\begin{figure}[t!]\centering
	\hspace{0.1em}
	\begin{minipage}[c]{0.48\linewidth}\centering
		\includegraphics[width=1\linewidth]{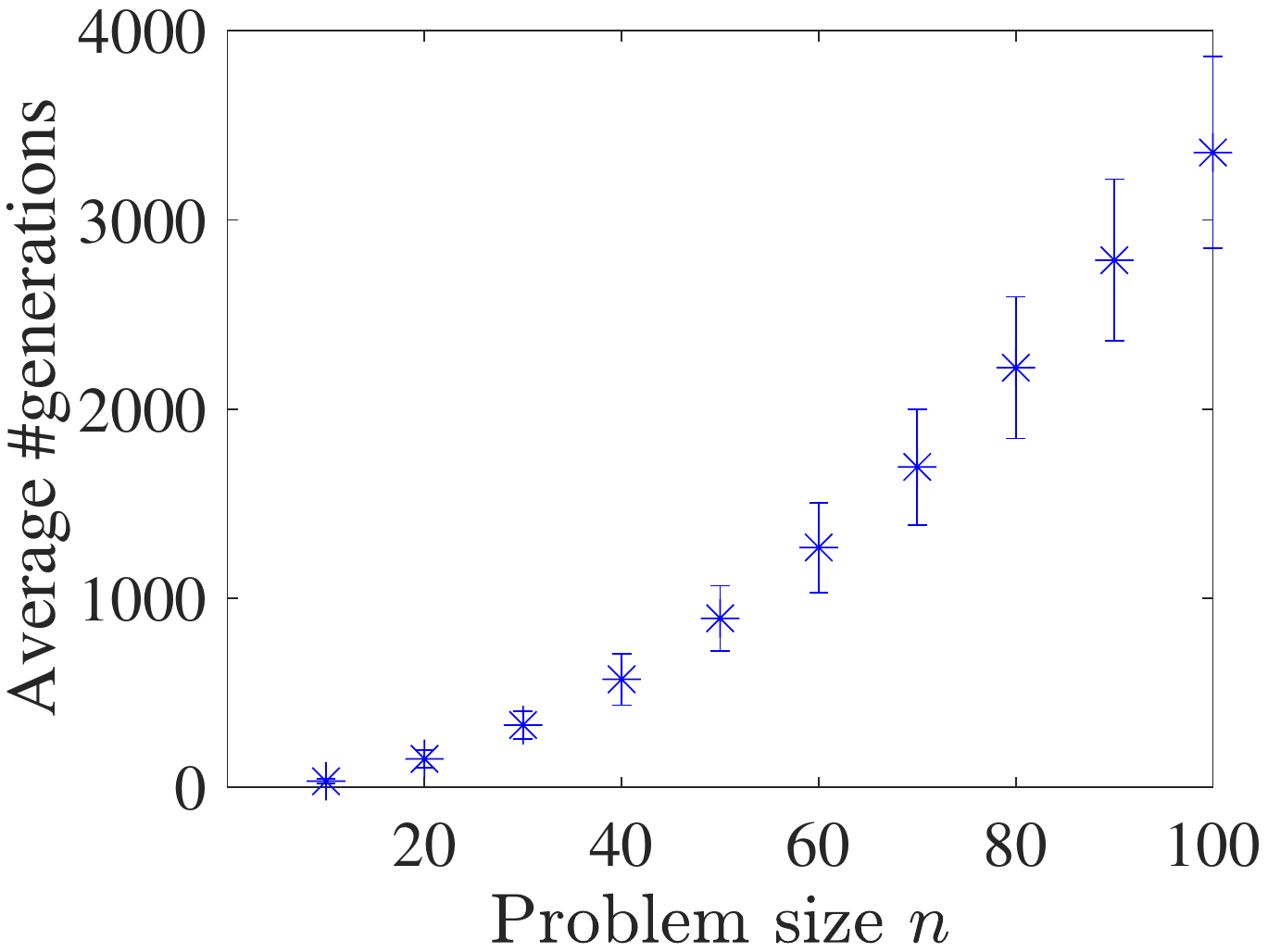}
	\end{minipage}
	\begin{minipage}[c]{0.48\linewidth}\centering
		\includegraphics[width=1\linewidth]{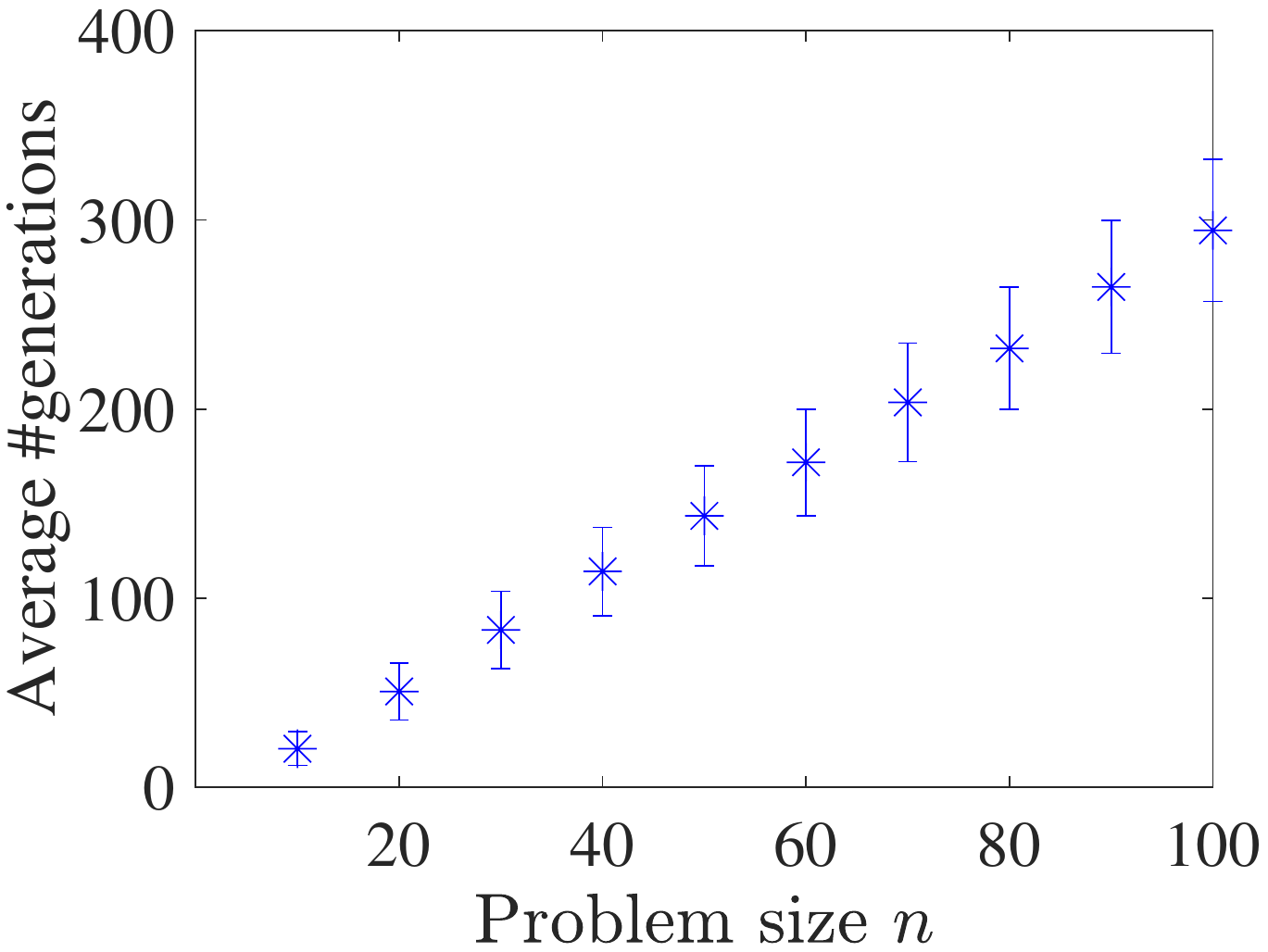}
	\end{minipage}\\\vspace{0.3em}
	\begin{minipage}[c]{1\linewidth}\centering
		\small(a) \text{LOTZ}
	\end{minipage}\\\vspace{1em}
	\begin{minipage}[c]{0.48\linewidth}\centering
		\includegraphics[width=1\linewidth]{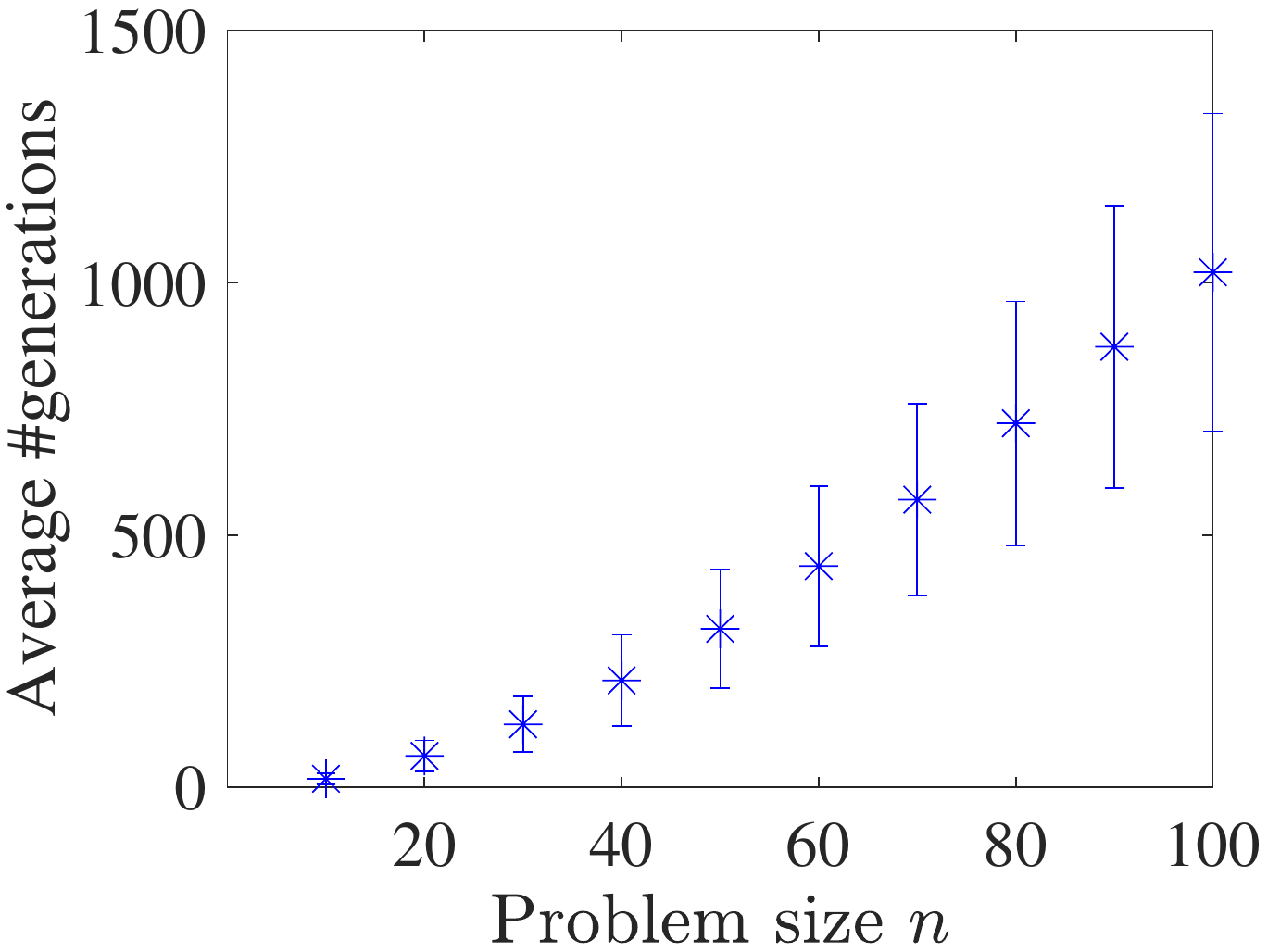}
	\end{minipage}
	\begin{minipage}[c]{0.48\linewidth}\centering
		\includegraphics[width=1\linewidth]{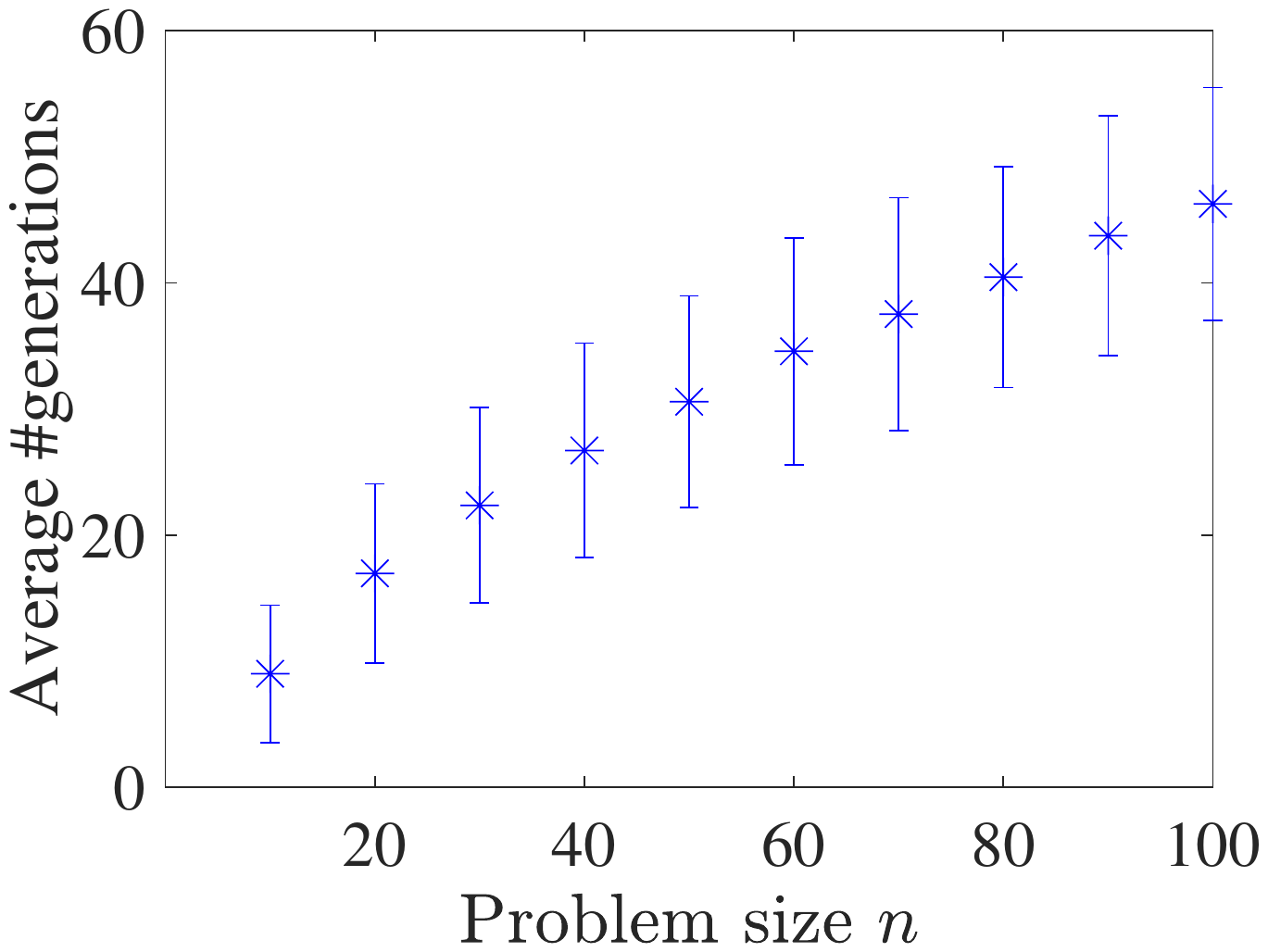}
	\end{minipage}\\\vspace{0.3em}
	\begin{minipage}[c]{1\linewidth}\centering
		\small(b) \text{OneMinMax}
	\end{minipage}\\\vspace{1em}
	\begin{minipage}[c]{0.48\linewidth}\centering
		\includegraphics[width=1\linewidth]{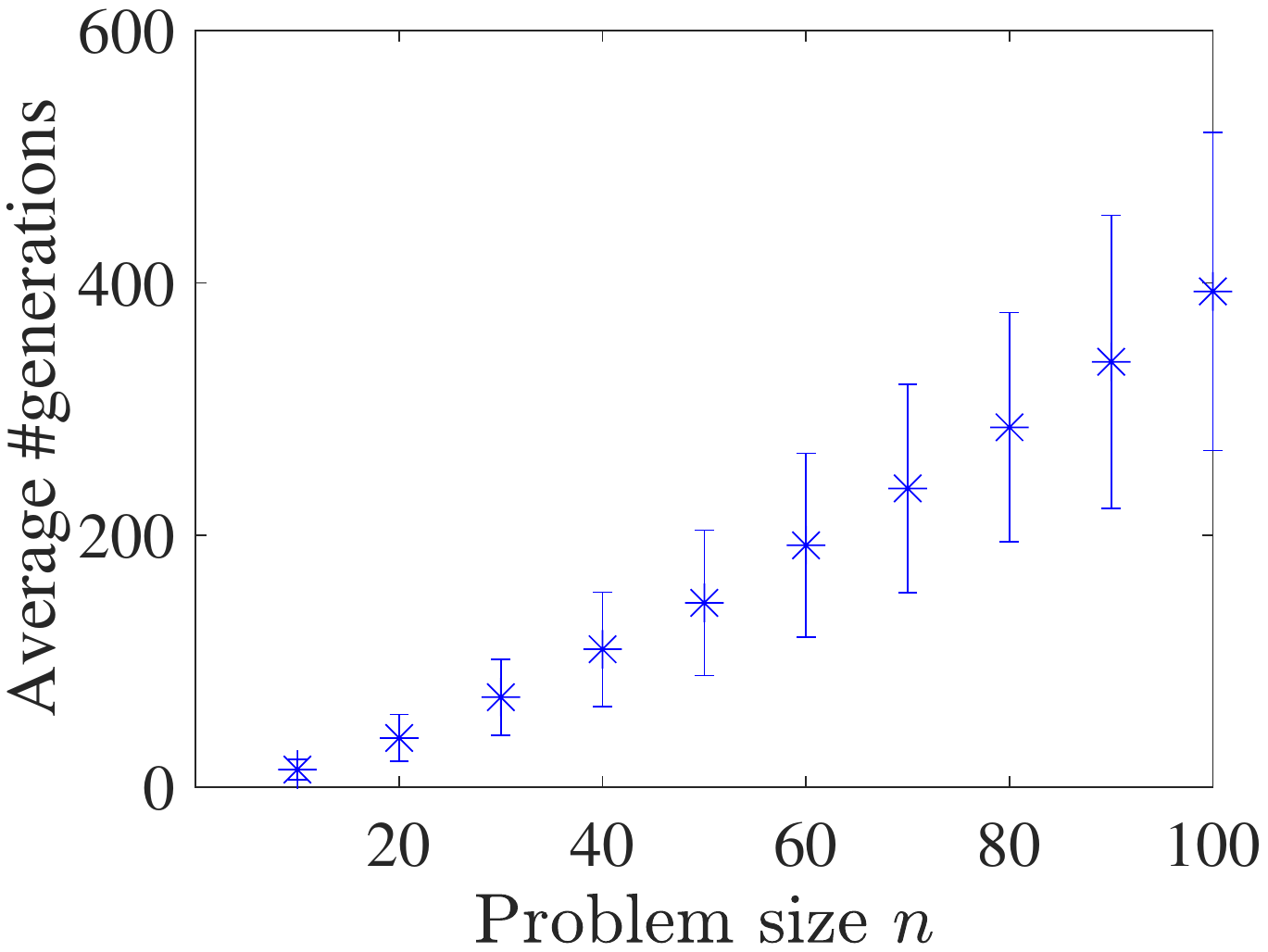}
	\end{minipage}
	\begin{minipage}[c]{0.48\linewidth}\centering
		\includegraphics[width=1\linewidth]{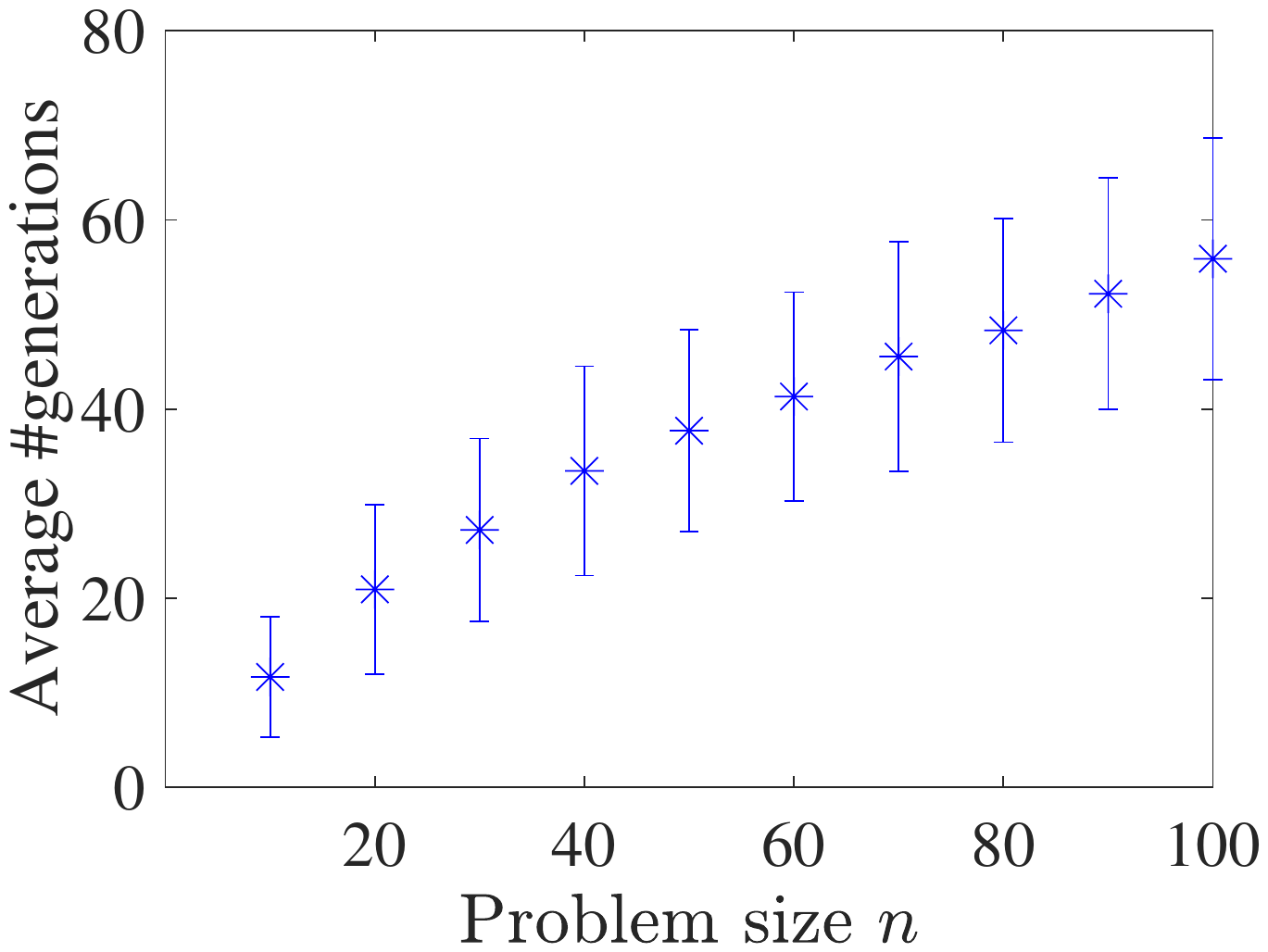}
	\end{minipage}\\\vspace{0.3em}
	\begin{minipage}[c]{1\linewidth}\centering
		\small(c) \text{COCZ}
	\end{minipage}\\\vspace{1em}
	\caption{Average \#generations of the NSGA-II 
		using binary tournament selection or stochastic tournament selection 
		for solving the LOTZ, OneMinMax and COCZ  problems.
		Left subfigure: average \#generations of the NSGA-II using binary tournament selection vs. problem size $n$;
		right subfigure: average \#generations of the NSGA-II using stochastic tournament selection vs. problem size $n$.}\label{fig:nsgaii-two-methods}
\end{figure}

 In the previous sections, we have proved that when binary tournament selection is used in the NSGA-II, the expected number of generations is $O(n^2)$ for LOTZ, and $O(n\log n)$ for OneMinMax and COCZ; when stochastic tournament selection is used, the expected number of generations can be improved to $O(n)$ for all the three problems. 
  But as the lower bounds on the running time have not been derived, the comparison may be not strict. Thus, we conduct experiments to examine the tightness of these upper bounds.
  
  For each problem,  we examine the performance of NSGA-II when the problem size $n$ changes from 10 to 100, with a step of 10. On each problem size $n$, we run the NSGA-II 1000 times independently, and record the number of generations until the Pareto front is found. Then, the average number of generations and the standard deviation of the 1000 runs are reported in Figure~\ref{fig:nsgaii-two-methods}.

From the left subfigure of Figure~\ref{fig:nsgaii-two-methods}(a), we can observe that the average number of generations increases by a factor of nearly four when the problem size $n$ doubles.  Thus, the average number of generations is approximately $\Theta(n^2)$,  implying that the upper bound $O(n^2)$ derived in Theorem~\ref{thm:bintour-lotz} is tight. By the right subfigure of Figure~\ref{fig:nsgaii-two-methods}(a),  the average number of generations is clearly a  linear function of $n$, which implies that the upper bound $O(n)$ derived in Theorem~\ref{thm:stotour-lotz} is also tight.  As the problem size $n$ increases, the average number of generations in the left subfigures of Figure~\ref{fig:nsgaii-two-methods}(b) and Figure~\ref{fig:nsgaii-two-methods}(c) both increases at a faster pace, implying that the expected number of generations is both $\omega(n)$, and thus the upper bound $O(n\log n)$ derived in Theorems~\ref{thm:bintour-omm} and~\ref{thm:bintour-cocz} is almost tight. 
From the right subfigures of Figure~\ref{fig:nsgaii-two-methods}(b) and Figure~\ref{fig:nsgaii-two-methods}(c), we can observe that the average number of generations increases by about 40\% when the problem size $n$ doubles, suggesting that the expected number of generations is approximately $\Omega(n^{1/2})$.

\section{Conclusion}
 In this paper, we theoretically study the running time of the NSGA-II solving three bi-objective problems, LOTZ, OneMinMax and COCZ,  and derive upper bounds that are as same as that of the previously analyzed simple MOEAs, GSEMO and SEMO.  Then, we propose a new parent selection strategy, stochastic tournament selection,  to replace the binary tournament selection strategy of the NSGA-II, and prove that the NSGA-II using the new strategy can find the Pareto front  of the three problems with much less time. Experiments are also conducted to examine the tightness of the upper bonds. 
 In the future, we will analyze the lower bounds on the running time 
 to make the comparison strict, and also conduct a more comprehensive experiment (e.g., increase the problem size $n$ to 500) to better reflect the tendency of the running time.
 It is also interesting and expected to study the running time of the NSGA-II on multi-objective combinatorial optimization problems.

\bibliographystyle{plainnat}
\bibliography{nsgaii-theory-bib}

\end{document}